\def\year{2022}\relax
\documentclass[letterpaper]{article} 
\usepackage{aaai22}  
\usepackage{times}  
\usepackage{helvet}  
\usepackage{courier}  
\usepackage[hyphens]{url}  
\usepackage{graphicx} 
\usepackage{natbib}  
\usepackage{caption} 
\DeclareCaptionStyle{ruled}{labelfont=normalfont,labelsep=colon,strut=off} 
\usepackage{amsmath,amsthm,amssymb}
\usepackage{makecell}
\usepackage{enumitem}
\usepackage{multirow}
\usepackage{booktabs}
\usepackage[ruled,lined,linesnumbered]{algorithm2e}
\usepackage{xcolor}

\newtheorem{theorem}{Theorem}

\title{Context-aware Health Event Prediction via Transition Functions on Dynamic Disease Graphs}
\author{
    Chang Lu, Tian Han, Yue Ning
}
\affiliations{
    Stevens Institute of Technology\\
    
    \{clu13, tian.han, yue.ning\}@stevens.edu
}

\newcommand{\retain}{{RETAIN}}
\newcommand{\deepr}{{Deepr}}
\newcommand{\gram}{{GRAM}}
\newcommand{\dipole}{{Dipole}}
\newcommand{\timeline}{{Timeline}}

\newcommand{\gbert}{{G-BERT}}
\newcommand{\hitanet}{{HiTANet}}
\newcommand{\cgl}{{CGL}}
\def\modelname{{Chet}}

\begin{document}

\maketitle

\begin{abstract}
   With the wide application of electronic health records (EHR) in healthcare facilities, health event prediction with deep learning has gained more and more attention. A common feature of EHR data used for deep-learning-based predictions is historical diagnoses. Existing work mainly regards a diagnosis as an independent disease and does not consider clinical relations among diseases in a visit. Many machine learning approaches assume disease representations are static in different visits of a patient. However, in real practice, multiple diseases that are frequently diagnosed at the same time reflect hidden patterns that are conducive to prognosis. Moreover, the development of a disease is not static since some diseases can emerge or disappear and show various symptoms in different visits of a patient. To effectively utilize this combinational disease information and explore the dynamics of diseases, we propose a novel context-aware learning framework using transition functions on dynamic disease graphs. Specifically, we construct a global disease co-occurrence graph with multiple node properties for disease combinations. We design dynamic subgraphs for each patient's visit to leverage global and local contexts. We further define three diagnosis roles in each visit based on the variation of node properties to model disease transition processes. Experimental results on two real-world EHR datasets show that the proposed model outperforms state of the art in predicting health events.
\end{abstract}


\maketitle

\section{Introduction}
\label{sec:intro}
Electronic health records (EHR) have been widely applied in healthcare facilities as a system to record patients' visit information. EHR provide valuable data sources for researchers to predict health events, such as diagnosis~\cite{choi2017gram}, mortality~\cite{darabi2019taper}, and readmission~\cite{Phuoc2017deepr}. These kinds of health event predictions are beneficial to both healthcare providers and patients to achieve preventative health monitoring and personalized care plans. A common approach in deep-learning-based health event predictions is mining temporal features of a patient, especially patterns of historical diagnoses, to predict future risks. Although deep learning models such as recurrent neural networks (RNN)~\cite{bai2018interpretable}, convolutional neural networks (CNN)~\cite{Phuoc2017deepr}, and attention-based models~\cite{shang2019pretrain} have shown great success in health event predictions, there are still several challenges in learning with diagnosis features.

\begin{figure}
    \centering
    \includegraphics[width=0.9\linewidth]{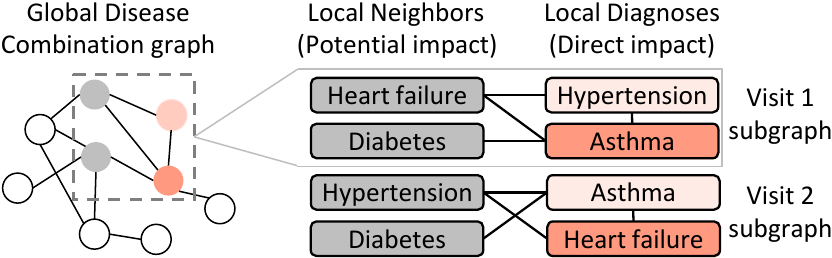}
    \caption{An example of a disease combination graph of an EHR dataset and subgraphs of two visits. Nodes are diseases. Edges denote disease co-occurrence. The red and gray nodes represent diagnoses and neighbors in a visit. White nodes are other diseases. Darker red nodes denote higher-priority diagnoses.}
    \label{fig:disease_example}
\end{figure}

\paragraph{Q1: How to effectively utilize disease combination information?} In medical practice, disease combinations refer to a group of diseases that are diagnosed in the same patient such as hypertension and heart failure. Disease combinations appearing in all patients' visits naturally form a global graph structure, as shown in \figurename~{\ref{fig:disease_example}} (left). This structure reflects hidden patterns among diseases. A visit can be regarded as local subgraphs containing diagnoses and their neighbors in the combination graph. Neighbor nodes refer to disease that are not diagnosed in the current visit for the current patient, but diagnosed in other patients' visits. The visit subgraphs enable us to estimate prognosis for future visits from neighbors. For example, in \figurename~{\ref{fig:disease_example}} (right), we may successfully predict heart failure in visit 2 given the relations among heart failure, hypertension, and asthma in the combination graph. However, this graph structure is not utilized in popular deep learning models for health event predictions including \gram~\cite{choi2017gram}, \timeline~\cite{bai2018interpretable}, and  \gbert~\cite{shang2019pretrain}.

\paragraph{Q2: How to explore the dynamic scheme of diseases?} During the development stage of a disease, the impact of this disease on a patient may not be static. EHR datasets like MIMIC-III~\cite{mimiciii} provide an indicator for diagnosis priority in each visit. It is common that the same diagnosis has different priorities in different visits. Moreover, even if a disease does not appear in a previous visit, it can also be diagnosed in future visits, especially when it is a neighbor of an existing diagnosis in the disease combination graph. As a result, a semantic context for reflecting the development scheme of diseases is implied by diagnoses and neighbors in visit sequences and the transition from neighbors to diagnoses. \figurename~\ref{fig:disease_example} (right) shows an example of diagnoses and neighbors in a disease combination graph with different impacts in two visits. In visit 1, asthma is the major diagnosis with the highest priority, while the major diagnosis becomes heart failure in visit 2 since it has more severe symptoms. In addition, we can observe that heart failure turns into a diagnosis in visit 2 from a neighbor in visit 1. It indicates that neighbors in a visit can also have a potential impact on patients in future visits. Therefore, it is necessary to explore the development scheme of diseases by dynamically representing diseases in different visits and learning the transition process of diseases from neighbors to diagnoses.

To address such challenges, we propose {\modelname}, a novel \textbf{c}ontext-aware \textbf{h}ealth \textbf{e}vent prediction framework via \textbf{t}ran-sition functions on dynamic disease graphs. To utilize disease combination information, we first construct a weighted global disease combination graph with multiple node properties based on the historical diagnoses of all patients. For each visit of a patient, we design dynamic subgraphs to integrate local context from diseases in this visit and global context from the entire EHR dataset. Then, to explore the development scheme of diseases, for each visit, we define multiple
diagnosis roles based on the variation of node properties in dynamic subgraphs to represent disease transition processes. We further design corresponding transition functions for each role to extract historical contexts. Finally, we integrate all visits of a patient and adopt an attention-based method to predict future health events. The main contributions of this work are summarized as follows:
\begin{itemize}
    \item We design a context-aware dynamic graph learning method to learn disease combinations. The proposed global disease graph and visit subgraphs can integrate global and local context from disease combinations.
    \item We propose a disease-level temporal learning to explore disease development schemes. Three diagnosis roles and corresponding transition functions can extract historical context and learn the transition process of diseases.
    \item We conduct comprehensive experiments on two real-world EHR datasets to show the improvement of {\modelname} over the state-of-the-art models on prediction accuracy.
\end{itemize}

\section{Problem Formulation}
\label{sec:method_problem}
EHR contain temporal patient records of visits to health facilities. A key record type in EHR is diagnosis. In each visit, a patient is diagnosed with one or multiple diseases, represented by medical codes. The medical codes are typically predefined by modern disease classification systems, such as ICD-9-CM or ICD-10. For example, ``left heart failure'' has a code of 428.1 in ICD-9-CM. In an EHR dataset, we denote the code collection as $\mathcal{C} = \{ c_1, c_2, \dots, c_d \}$, where $d$ is the code number. For a patient $u \in \mathcal{U}$, the diagnoses in the $t$-th visit are defined as a multi-hot column vector $\mathbf{m}^t \in \{0, 1\}^{d}$. Here, $\mathcal{U}$ is the patient collection in the EHR dataset, $\mathbf{m}^t_i = 1$ means $u$ is diagnosed with $c_i$ in $t$-th visit, and $t = 1, 2, \dots, T$, where $T$ is the visit number of $u$.

\paragraph{EHR dataset.}
Let $r_u = (\mathbf{m}^1, \mathbf{m}^2, \dots, \mathbf{m}^T)$ be a visit sequence containing patient $u$'s diagnoses. The EHR dataset is defined as $\mathcal{D} = \{ r_u \mid u \in \mathcal{U} \}$.

\paragraph{Health event prediction.}
Given an EHR dataset $\mathcal{D}$, a patient $u$ and $u$'s previous diagnoses $r_u$, health event prediction is to predict an event $\mathbf{y}^{T + 1}$ of the future visit $T + 1$.

Common health event predictions include diagnosis prediction and heart failure prediction. For example, the ground-truth of diagnosis prediction  in the visit $T + 1$ is the diagnoses $\mathbf{y}^{T + 1} \in \{0, 1\}^d$.

\section{Methodology}
\label{sec:method}

\begin{figure}
    \centering
    \includegraphics[width=\linewidth]{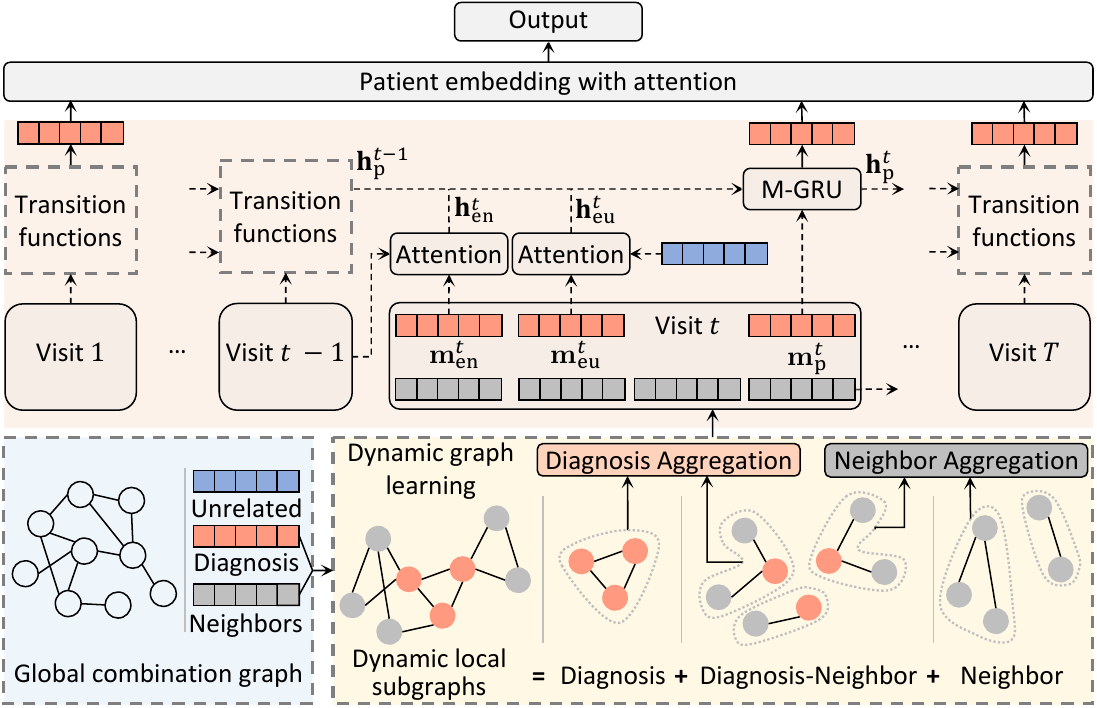}
    \caption{An overview of the proposed model. It includes a global combination graph for all diseases, a dynamic graph learning module for a visit, a dynamic temporal learning with transition functions for all visits, and an attention methods to calculate the patient embedding.}
    \label{fig:system}
\end{figure}

In this section, we demonstrate the details of {\modelname}. An overview of {\modelname} is shown in \figurename~\ref{fig:system}.

\subsection{Context-aware Dynamic Graph Learning}
\label{sec:method_graph}
In healthcare, it is common that a patient is diagnosed with a specific combination of diseases, such as diabetes and hypertension. This combination occurs because diseases share similar causes and risk factors. In practice, disease combinations are typically reflected by their co-occurrence. Therefore, we create a global co-occurrence graph $\mathcal{G}$ for all diseases with weighted edges. A node of $\mathcal{G}$ is a code $c$ in $\mathcal{C}$. If a code pair $(c_i, c_j)$ co-occur in a patient' visit record, we add two edges $\overrightarrow{(i, j)}$ and $\overleftarrow{(i, j)}$ with different weights to $\mathcal{G}$. Then we count the total co-occurrence frequency $f_{ij}$ of $(c_i, c_j)$ in all patients' visits for further calculation of edge weights. Here, we use two weighted edges for each pair of nodes because we conjecture that two diseases should not have equal influence on each other. One disease can be a common combination of the other, but not vice versa. In addition, we want to detect important and common disease pairs. Therefore, we define a threshold $\delta$ to filter out combinations with low frequency and get a qualified set $\boldsymbol{\Delta}_i = \{c_j \mid \frac{f_{ij}}{\sum_{j=1}^{d}{f_{ij}}} \ge \delta\}$ for $c_i$. Let $q_i = \sum_{c_j \in \boldsymbol{\Delta}_i}{f_{ij}}$ be the total frequency of qualified diseases co-occurred with $c_i$, we use an adjacency matrix $\mathcal{A} \in \mathbb{R}^{d \times d}$ to represent $\mathcal{G}$:
\begin{align}
    \mathcal{A}_{ij} = \begin{cases}
        \hfil 0 &~\text{if $i = j$ or $\frac{f_{ij}}{\sum_{j=1}^{d}{f_{ij}}} < \delta$}, \\
        \frac{f_{ij}}{q_i} &~\text{otherwise.} \\
    \end{cases}
\end{align}
Note that $\mathcal{A}$ is designed to be asymmetric to represent different influence of two diseases. As a static graph, $\mathcal{A}$ measures global co-occurrence frequencies of diseases. However, diseases may emerge and disappear at different stages of a treatment. Even though a disease is not diagnosed in the current visit, it may appear in the future due to the development of an existing diagnosis. Therefore, for each visit $t$, we consider three dynamic subgraphs of $\mathcal{G}$:
\begin{itemize}
    \item \textbf{Local diagnosis graph} $\mathcal{G}^t_{D}$. It is a complete graph consisting of diagnoses in visit $t$. We use an adjacency matrix $\mathcal{M}^t \in \mathbb{R}^{d \times d}$ to represent $\mathcal{G}^t_{D}$. $\mathcal{M}^t_{ij} = \mathcal{A}_{ij}$ if $c_i$ and $c_j$ are diagnosed in  visit $t$. Otherwise $\mathcal{M}^t_{ij} = 0$.
    \item \textbf{Global diagnosis-neighbor graph} $\mathcal{G}^t_{DN}$. It is a bipartite graph describing the connection of diagnoses in visit $t$ and their neighbors (connected diseases but not diagnosis in visit $t$) in $\mathcal{G}$. We use an adjacency matrix $\mathcal{B}^t \in \mathbb{R}^{d \times d}$ to represent connections from diagnoses to neighbors. $\mathcal{B}^t_{ij} = \mathcal{A}_{ij}$ if $c_i$ is diagnosed in  visit $t$ and $c_j$ is a neighbor of $c_i$ in $\mathcal{G}$ and $c_j$ is not diagnosed in  visit $t$. Otherwise $\mathcal{B}^t_{ij} = 0$. For opposite connections, since $\mathcal{A}$ is asymmetric, the transpose of $\mathcal{B}^t$ cannot be used. We use another matrix $\hat{\mathcal{B}}^t$ instead to denote neighbor-diagnosis edges. It is calculated in a similar manner to $\mathcal{B}^t_{ij}$.
    \item \textbf{Global neighbor graph} $\mathcal{G}^t_N$. It is a graph for neighbors in visit $t$. We use an adjacency matrix $\mathcal{N}^t \in \mathbb{R}^{d \times d}$ to represent $\mathcal{G}^t_{N}$. $\mathcal{N}^t_{ij} = \mathcal{A}_{ij}$ if $c_i$ and $c_j$ are two neighbors and are not diagnosed in visit $t$. Otherwise $\mathcal{N}^t_{ij} = 0$.
\end{itemize}
It is worth noting that the corresponding subgraphs of two visits are different unless the two visits have all the same diagnoses. Moreover, for one diagnosis, if it appears in two visits, its neighbors in these two visits can also be different, because some neighbors in a visit may be diagnosed in the other visit and cannot be neighbors.

As we discussed above, disease combinations reflect important disease relationships. In each visit, a diagnosis and its neighbors can both have an influence on future visits. Therefore, we use a graph neural network (GNN) to learn the combinations. An intuitive way is to assign each disease with a universal embedding vector. However, when a disease is a diagnosis, it has a direct impact on this patient, such as causing various symptoms. When this disease is a neighbor, it may only have a hidden impact on future visits of this patient. Therefore, we use two embedding matrices $\mathbf{M}, \mathbf{N} \in \mathbb{R}^{d \times s}$ with size $s$ to represent diseases when they are diagnoses and neighbors, respectively. For diseases not appearing in current diagnoses nor their direct neighbors, we name them unrelated diseases and use another embedding matrix $\mathbf{R} \in \mathbb{R}^{d \times s'}$ to represent them. We conjecture that their impact on future visits is less than direct neighbors.

In a graph layer, we extract both local and global contexts for diagnoses and neighbors in visit $t$:
\begin{itemize}
    \item \textbf{Local context.} For each diagnosis node, we aggregate other diagnoses' embeddings from $\mathcal{G}^t_{D}$ as local context.
    \item \textbf{Diagnosis global context.} For each diagnosis node, We aggregate the embeddings of connected neighbors from $\mathcal{G}^t_{DN}$ as diagnosis global context.
    \item \textbf{Neighbor global context.} For each neighbor node, we aggregate the embeddings of connected diagnosis nodes from $\mathcal{G}^t_{DN}$ and aggregate connected neighbor nodes from $\mathcal{G}^t_N$ as neighbor global context.
\end{itemize}
Then, for each node, we add the corresponding context to the embedding of this node as message aggregation in GNN. Let $\mathbf{n}^t \in \{0, 1\}^d$ be a multi-hot vector which denotes all neighbors in visit $t$. The graph layer is summarized as follows:
\begin{gather}
\label{eq:graph1}
    \mathbf{Z}^t_D = \mathbf{m}^t \odot \mathbf{M} + \underbrace{\mathcal{M}^t\mathbf{M}}_{\substack{\text{Local context}}} + \underbrace{\mathcal{B}^t\mathbf{N}}_{\substack{\text{Diagnosis} \\ \text{global context}}} \in \mathbb{R}^{d \times s}, \\
\label{eq:graph2}
    \mathbf{Z}^t_N = \mathbf{n}^t \odot \mathbf{N} + \underbrace{\mathcal{N}^t\mathbf{N} + \hat{\mathcal{B}}^t\mathbf{M}}_{\substack{\text{Neighbor global context}}} \in \mathbb{R}^{d \times s}.
\end{gather}
Here, $\odot$ denotes the element-wise multiplication. $\mathbf{m}^t \odot \mathbf{M}\in \mathbb{R}^{d\times s}$ selects the rows in $\mathbf{M}$ that are corresponding to the diagnoses in $\mathbf{m}^t $ and sets other rows to be 0. Finally, in visit $t$, the hidden embeddings of diagnoses and neighbors, i.e., GNN outputs, are calculated with a fully connected layer:
\begin{align}
    \label{eq:graph_act}
    \mathbf{H}^t_{\{D,N\}} = \text{LeakyReLU}\left(\mathbf{Z}^t_{\{D,N\}}\mathbf{W}\right) \in \mathbb{R}^{d \times s'}.
\end{align}
Here, $\mathbf{W} \in \mathbb{R}^{s \times s'}$ is a weight matrix. We adopt LeakyReLU \cite{xu2015empirical} as the activation function. Note that, we only use one graph layer because only one-hop connections are considered in visit subgraphs.

\subsection{Subgraphs' Adjacency Matrix Calculation}
\label{sec:method_mem}
In our settings, the dimension of $\mathcal{M}^t, \mathcal{B}^t$, $\hat{\mathcal{B}}^t$, and $\mathcal{N}^t$ is ${d \times d}$. In real practice, there may exist storage/memory problems in Equations~(\ref{eq:graph1}) and (\ref{eq:graph2}) even with sparse matrix:
\begin{itemize}
    \item \textbf{Storing the entire EHR dataset}. The dimension of the entire EHR dataset is $|\mathcal{U}| \times T \times d \times d$, if we store adjacency matrices for all visits. When $|\mathcal{U}|$ and $d$ increase, the size of the dataset will increase rapidly.
    \item \textbf{Training models with mini-batches}. When training a deep learning model, assuming the batch size is $b$, we have to load four batched matrices, $\mathcal{M}^t, \mathcal{B}^t$, $\hat{\mathcal{B}}^t$, and $\mathcal{N}^t$ whose dimensions are $b \times T \times d \times d$, into CPU/GPU memory, but it is usually not applicable when $d$ is large.
\end{itemize}
To address these problems, we replace the four matrices with the diagnosis/neighbor vectors $\mathbf{m}^t, \mathbf{n}^t$ and the static adjacency matrix $\mathcal{A}$. The memory-efficient calculation for Equations~(\ref{eq:graph1}) and (\ref{eq:graph2}) is summarized as follows:
\begin{align}
    \mathbf{Z}^t_D &= \mathbf{m}^t \odot (\mathbf{M} + \mathcal{A}(\mathbf{m}^t \odot \mathbf{M}) + \mathcal{A}(\mathbf{n}^t \odot \mathbf{N})), \\
    \mathbf{Z}^t_N &= \mathbf{n}^t \odot (\mathbf{N} + \mathcal{A}(\mathbf{n}^t \odot \mathbf{N}) + \mathcal{A}(\mathbf{m}^t \odot \mathbf{M})).
\end{align}
This optimization reduces the dimension of $d \times d$ to $d \times s$ for the four matrices, where $s$ is the embedding size for diseases and $s \ll d$.  Detailed derivation and proofs can be found in Appendix: Subgraphs' Adjacency Matrix Calculation.

\subsection{Disease-level Temporal Learning with Transition Functions}
\label{sec:temporal}
A core function of deep learning models for healthcare is exploring temporal features for visit sequences to learn previous diagnoses and predict future events. For example, if a patient has a chronic disease like chronic heart failure in a visit, it is highly possible that this disease will last for a long term and be diagnosed in future visits. In addition, even if a disease is not diagnosed in a visit, it may also appear resulting from an existing diagnosis. For example, longstanding hypertension can ultimately lead to heart failure~\cite{messerli2017transition}. Therefore, for the diagnosis vector $\mathbf{m}^t$ in visit $t$ when $t \ge 2$, we further define three diagnosis roles by dividing $\mathbf{m}^t$ into three disjoint parts to represent longstanding and new-onset diseases:
\begin{enumerate}
    \item \textbf{Persistent diseases} $\mathbf{m}^t_{\text{p}} = \mathbf{m}^t \wedge \mathbf{m}^{t - 1} \in \{0, 1\}^d$: Diagnoses in visit $t$ that are also diagnoses in visit $t - 1$.
    \item \textbf{Emerging neighbors} $\mathbf{m}^t_{\text{en}} = \mathbf{m}^t \wedge \mathbf{n}^{t - 1} \in \{0, 1\}^d$: Diagnoses in visit $t$ that are neighbors in visit $t - 1$.
    \item \textbf{Emerging unrelated diseases} $\mathbf{m}^t_{\text{eu}} = \mathbf{m}^t \wedge \lnot(\mathbf{m}^{t - 1} \lor \mathbf{n}^{t - 1})$ $\in \{0, 1\}^d$: Diagnoses in visit $t$ that are unrelated diseases in visit $t - 1$.
\end{enumerate}
Here, $\wedge, \lor, \lnot$ are element-wise logical conjunction, disjunction, and negation for multi-hot vectors.

In temporal learning, previous methods mainly think diseases are static. They aggregate diagnoses into a visit embedding~\cite{choi2017gram, bai2018interpretable} and use it for temporal learning with an RNN or attention-based method. However, we think that even the same disease in different visits can have different influence on a patient. Therefore, instead of using visit embeddings, we choose to perform temporal learning on diagnoses with the three parts of $\mathbf{m}^t$ and design three \textbf{transition functions} corresponding to each part to extract \textbf{historical context} from previous visits.

For emerging diseases $\mathbf{m}^t_{\text{en}}$ and $\mathbf{m}^t_{\text{eu}}$, they are not continuous transitions and do not directly inherit information from the last diagnoses. We design a scaled dot-product attention~\cite{Vaswani2017attention} as two transition functions to calculate transition outputs $\mathbf{h}_{\text{en}}^t, \mathbf{h}_{\text{eu}}^t \in \mathbb{R}^{d \times p}$, where $p$ is the dimension after attention.
\begin{align}
    \label{eq:trans_1}
    \mathbf{h}_{\text{en}}^t &= \text{Attn}(\mathbf{m}_{\text{en}}^t \odot \mathbf{H}^{t-1}_N, \mathbf{m}_{\text{en}}^t \odot \mathbf{H}^{t-1}_N, \mathbf{m}_{\text{en}}^t \odot \mathbf{H}^t_D), \\
    \label{eq:trans_2}
    \mathbf{h}_{\text{eu}}^t &= \text{Attn}(\mathbf{m}_{\text{eu}}^t \odot \mathbf{R}, \mathbf{m}_{\text{eu}}^t \odot \mathbf{R}, \mathbf{m}_{\text{eu}}^t \odot \mathbf{H}^t_D).
\end{align}
The attention in the above equations is defined as follows:
\begin{align}
    \text{Attn}(\mathbf{Q}, \mathbf{K}, \mathbf{V}) = \text{softmax}\left( \frac{\mathbf{Q}\mathbf{W}_q(\mathbf{K}\mathbf{W}_k)^\top}{\sqrt{a}} \right)\mathbf{V}\mathbf{W}_v. \notag
\end{align}
Here, $a$ is the attention size. $\mathbf{W}_q, \mathbf{W}_k \in \mathbb{R}^{s' \times a}, \mathbf{W}_v \in \mathbb{R}^{s' \times p}$ are attention weights.
For $\mathbf{h}_{\text{en}}^t$, we use hidden neighbor embeddings $\mathbf{H}^{t-1}_N$ from Equation~(\ref{eq:graph_act}) as query $\mathbf{Q}$ and key $\mathbf{K}$. For $\mathbf{h}_{\text{eu}}^t$, we use universal embeddings of unrelated diseases as query and key. For both equations, we use hidden diagnosis embeddings $\mathbf{H}^t_D$ as value $\mathbf{V}$ in attention.

For persistent diseases $\mathbf{m}^t_{\text{p}}$, they directly inherit information from previous diagnoses. We design a modified gated recurrent unit (M-GRU) as the transition function to model continuous features. In visit $t$, the transition output $\mathbf{h}^t_\text{p} \in \mathbb{R}^{d \times p}$ is calculated by M-GRU using hidden embeddings $\mathbf{H}^t_D$ of diagnoses in visit $t$, transition outputs $\mathbf{h}_{\text{en}}^t, \mathbf{h}_{\text{eu}}^t$ of emerging diseases in visit $t$, and the hidden state as well as the M-GRU output $\mathbf{h}^{t - 1}_\text{p} \in \mathbb{R}^{d \times p}$ in visit $t - 1$:
\begin{align}
\label{eq:gru}
    \mathbf{h}^t_\text{p} = \text{M-GRU}\left(\mathbf{m}_\text{p}^t \odot \mathbf{H}^t_D, \mathbf{h}_{\text{en}}^t, \mathbf{h}_{\text{eu}}^t, \mathbf{h}^{t - 1}_\text{p}\right).
\end{align}
Since the original GRU~\cite{cho-etal-2014-learning} is designed for a vector input, we rewrite GRU into a matrix version. Moreover, the persistent diseases in visit $t + 1$ can also be emerging neighbor/unrelated diseases in visit $t$. Therefore, we also need to store the hidden state of emerging diseases in visit $t$ for future usage. The detailed M-GRU in Equation~(\ref{eq:gru}) is summarized as follows:
\begin{align}
    \mathbf{z}^t &= \sigma(\mathbf{m}_\text{p}^t \odot \mathbf{H}^t_D\mathbf{W}_z + \mathbf{h}^{t - 1}_\text{p}\mathbf{U}_z + \mathbf{b}_z), \label{eq:gru_1} \\
    \mathbf{r}^t &= \sigma(\mathbf{m}_\text{p}^t \odot \mathbf{H}^t_D\mathbf{W}_r + \mathbf{h}^{t - 1}_\text{p}\mathbf{U}_r + \mathbf{b}_r), \label{eq:gru_2} \\
    \hat{\mathbf{h}}^t &= \phi(\mathbf{m}_\text{p}^t \odot \mathbf{H}^t_D\mathbf{W}_h + (\mathbf{r}^t \odot \mathbf{h}^{t - 1}_\text{p})\mathbf{U}_h + \mathbf{b}_h), \label{eq:gru_3} \\
    \tilde{\mathbf{h}}^t &= \phi(\mathbf{h}_{\text{en}}^t + \mathbf{h}_{\text{eu}}^t), \label{eq:gru4}\\
    \mathbf{h}^t_\text{p} &= (1 - \mathbf{z}^t) \odot \mathbf{h}^{t - 1}_\text{p} + \mathbf{z}^t \odot \hat{\mathbf{h}}^t + \tilde{\mathbf{h}}^t. \label{eq:gru5}
\end{align}
Here, $\mathbf{W}_{\{z,r,h\}} \in \mathbb{R}^{s' \times p}$ and $\mathbf{U}_{\{z,r,h\}} \in \mathbb{R}^{p \times p}$ are GRU weights. $\mathbf{b}_{\{z,r,h\}} \in \mathbb{R}^p$ are bias. $\sigma$ and $\phi$ denote sigmoid and tanh activation functions. Equations~(\ref{eq:gru4}) and (\ref{eq:gru5}) are the main modification to the original GRU beside the matrix version. In Equation~(\ref{eq:gru4}), we also apply tanh by following the original GRU to the output of emerging diseases and calculate their hidden state $\tilde{\mathbf{h}}^t$. In Equation~(\ref{eq:gru5}), we store the hidden state of emerging diseases to the hidden state $\mathbf{h}^t_\text{p}$ of persistent diseases for the next visit. Note that, since $\mathbf{m}^t_{\{\text{p, en, eu}\}}$ are disjoint, the add operations in Equations~(\ref{eq:gru4}) and (\ref{eq:gru5}) do not add values to the same entry in $\mathbf{h}^t_{\{\text{p, en, eu}\}}$. In addition, if there are no persistent diseases or emerging diseases in visit $t$, we ignore the corresponding part of $\mathbf{h}^{t - 1}_\text{p}, \hat{\mathbf{h}}^t$, $\mathbf{h}_{\text{en}}^t$, or $\mathbf{h}_{\text{eu}}^t$ in Equations~(\ref{eq:gru4}) and (\ref{eq:gru5}).

When $t = 1$, since there are no emerging diseases yet in the first visit, we let $\mathbf{m}_\text{p}^{1} = \mathbf{m}^{1}$ and use the original GRU with an initial hidden state $\mathbf{h}^{0}_\text{p} = \mathbf{0}$ to calculate $\mathbf{h}^{1}_\text{p}$ and $\mathbf{v}^{1}$:
\begin{align}
\label{eq:visit1}
    \mathbf{h}^1_\text{p} &= \text{GRU}\left(\mathbf{m}_\text{p}^1 \odot \mathbf{H}^1_D, \mathbf{h}^0_\text{p}\right).
\end{align}
Note that, the original GRU and M-GRU share the same parameters in Equations~(\ref{eq:gru_1})-(\ref{eq:gru_3}). Then, we use max pooling for the transition output of the three parts and calculate the visit embedding $\mathbf{v}^t$. Since $\mathbf{h}_\text{p}^t$ contains all diagnoses in visit $t$ when $t = 1$ and $t \ge 2$, we employ $\mathbf{h}_\text{p}^t$ for max pooling:
\begin{align}
\label{eq:visit}
    \mathbf{v}^t = \text{max\_pooling}(\mathbf{h}^t_\text{p}) \in \mathbb{R}^{p}.
\end{align}

Finally, we apply a location-based attention~\cite{luong2015effective} to calculate the final hidden representation $\mathbf{o}$ of all visits, i.e., patient embedding:
\begin{align}
    \label{eq:attention}
    \boldsymbol{\alpha} &= \text{softmax}\left( [\mathbf{v}^1, \mathbf{v}^2, \dots, \mathbf{v}^T]\mathbf{W}_{\alpha} \right) \in \mathbb{R}^{T}, \\
    \mathbf{o} &= \boldsymbol{\alpha}[\mathbf{v}^1, \mathbf{v}^2, \dots, \mathbf{v}^T]^\top \in \mathbb{R}^p, \label{eq:attention3}
\end{align}
where $\mathbf{W}_{\alpha} \in \mathbb{R}^{p}$ is a context vector for attention and $\boldsymbol{\alpha}$ is the attention score for  visits. The patient embedding $\mathbf{o}$ will be used in a classifier for final predictions of a specific task.

\section{Experiments}
\label{sec:exp}

\subsection{Experimental Setups}
\label{sec:exp_setup}
\subsubsection{{Tasks}} We use two common tasks to predict health events:
\begin{itemize}
    \item \textbf{Diagnosis prediction.} This task predicts all medical codes, i.e. diagnoses of the visit $T + 1$ given previous $T$ visits. It is a multi-label classification.
    \item \textbf{Heart failure prediction.} This task predicts whether a patient will be diagnosed with heart failure in the visit $T + 1$ given previous $T$ visits.\footnote{The codes of heart failure start with 428 in ICD-9-CM.} It is a binary classification.
\end{itemize}
For both tasks, we use a fully connected layer with a sigmoid activation function as the classifier and a binary cross-entropy loss as the objective function.

The evaluation metrics for the diagnosis prediction are weighted $F_1$ score (w-$F_1)$~\cite{bai2018interpretable} and top $k$ recall (R@$k$)~\cite{choi2016doctor}. w-$F_1$ is a weighted sum of $F_1$ scores for all medical codes. It measures overall effectiveness of predictions on all classes. R@$k$ is an average ratio of desired medical codes in top $k$ predictions by the total number of desired medical codes in each visit. It measures prediction accuracy on a subset of classes.

The evaluation metrics for the heart failure prediction are the area under the receiver operating characteristic curve (AUC) and $F_1$ score since this task is a binary classification on imbalanced test data in our experiments.

\subsubsection{{Datasets}}
We use MIMIC-III~\cite{mimiciii} and MIMIC-IV~\cite{mimiciv} to validate the predictive power of {\modelname}. MIMIC-III contains 7,493 patients with multiple visits ($T \ge 2$) from 2001 to 2012, while there are 85,155 patients in MIMIC-IV with multiple visits from 2008 to 2019. Since there is an overlapped time range between MIMIC-III and MIMIC-IV, we randomly sample 10,000 patients from MIMIC-IV from 2013 to 2019. The statistics of MIMIC-III and MIMIC-IV are shown in Table~\ref{tab:dataset}.

We further randomly split the two datasets based on patients into training/validation/test sets, which contain 6,000/ 493/1,000 patients for MIMIC-III and 8,000/1,000/1,000 for MIMIC-IV, respectively. We use the last visits as labels and adopt the rest as features. The global combination graph $\mathcal{G}$ is constructed based on the feature visits in the training set. For the heart failure prediction task, we set the label as 1 if a patient is diagnosed with heart failure in the last visit. There are 36.70\% positive samples and 63.30\% negative samples in the test set of MIMIC-III, and 15.70\% positive samples and 85.30\% negative samples in the test set of MIMIC-IV.

\begin{table}
    \centering
    \small
    \begin{tabular}{lcccc}
        \toprule
        \textbf{Dataset} & \textbf{MIMIC-III} & \textbf{MIMIC-IV} \\
        \midrule
        \# patients             & 7,493  & 10,000 \\
        Max. \# visit       & 42     & 55 \\
        Avg. \# visit       & 2.66   & 3.66 \\
        \midrule
        \# codes                & 4,880  & 6,102  \\
        Max. \# codes per visit      & 39       & 50 \\
        Avg. \# codes per visit      & 13.06       & 13.38 \\
        \bottomrule
    \end{tabular}
    \caption{Statistics of MIMIC-III and MIMIC-IV datasets}
    \label{tab:dataset}
\end{table}

\begin{table*}
    \centering
    \small
    \begin{tabular}{lcccc|cccc}
        \toprule
        \multirow{2}{*}{\textbf{Models}} & \multicolumn{4}{c}{\textbf{MIMIC-III}} & \multicolumn{4}{c}{\textbf{MIMIC-IV}} \\ \cmidrule{2-9}
        & \multirow{1}{*}{\textbf{w-}$\boldsymbol{F_1}$} & \textbf{R@10}& \textbf{R@20}& \multirow{1}{*}{ \textbf{\# Params}} & \multirow{1}{*}{\textbf{w-}$\boldsymbol{F_1}$} & \textbf{R@10}& \textbf{R@20} & \multirow{1}{*}{ \textbf{\# Params}} \\ 
        \midrule
        \retain    & 20.69 (0.15) & 26.13 (0.18) & 35.08 (0.22) & 2.90M & 24.71 (0.24) & 28.02 (0.47) & 34.46 (0.13) & 3.56M \\
        \deepr     & 18.87 (0.21) & 24.74 (0.25) & 33.47 (0.17) & 1.16M & 24.08 (0.17) & 26.29 (0.25) & 33.93 (0.21) & 1.44M \\
        \gram      & 21.52 (0.10) & 26.51 (0.09) & 35.80 (0.09) & 1.59M & 23.50 (0.11) & 27.29 (0.27) & 36.36 (0.30) & 1.67M \\
        \dipole    & 19.35 (0.33) & 24.98 (0.27) & 34.02 (0.21) & 2.18M & 23.69 (0.24) & 27.38 (0.35) & 35.48 (0.29) & 2.51M \\
        \timeline  & 20.46 (0.18) & 25.75 (0.13) & 34.83 (0.14) & 1.23M & 25.26 (0.30) & 29.00 (0.21) & 37.13 (0.39) & 1.52M \\
        \gbert     & 19.88 (0.19) & 25.86 (0.12) & 35.31 (0.13) & 6.51M & 24.49 (0.20) & 27.16 (0.06) & 35.86 (0.19) & 7.53M \\
        \hitanet   & 21.15 (0.19) & 26.02 (0.25) & 35.97 (0.18) & 3.33M & 24.92 (0.27) & 27.45 (0.33) & 36.37 (0.24) & 3.96M \\
        \cgl       & 21.92 (0.12) & 26.64 (0.30) & 36.72 (0.15) & 1.53M & 25.41 (0.08) & 28.52 (0.42) & 37.15 (0.29) & 1.83M  \\
        \midrule
        \modelname & \textbf{22.63 (0.08)} & \textbf{28.64 (0.13)} & \textbf{37.87 (0.09)} & 2.12M & \textbf{26.35 (0.13)} & \textbf{30.28 (0.09)} & \textbf{38.69 (0.15)} & 2.59M \\
        \bottomrule
    \end{tabular}
    \caption{Diagnosis prediction results on MIMIC-III and MIMIC-IV using w-${F_1}$ (\%) and R@${k}$ (\%).}
    \label{tab:result_code}
\end{table*}

\begin{table*}
    \centering
    \small
    \begin{tabular}{lccc|ccc}
        \toprule
        \multirow{2}{*}{\textbf{Models}} & \multicolumn{3}{c}{\textbf{MIMIC-III}} & \multicolumn{3}{c}{\textbf{MIMIC-IV}} \\  \cmidrule{2-7}
        & \textbf{AUC} & $\boldsymbol{F_1}$ & \textbf{\# Params} & \textbf{AUC} & $\boldsymbol{F_1}$ & \textbf{\# Params} \\
        \midrule
        \retain   & 83.21 (0.26) & 71.32 (0.17) & 1.67M &  89.02 (0.26) & 67.38 (0.21) & 1.99M \\
        \deepr    & 81.36 (0.13) & 69.54 (0.08) & 0.53M &  88.43 (0.18) & 61.36 (0.12) & 0.65M \\
        \gram     & 83.55 (0.19) & 71.78 (0.14) & 0.96M &  89.61 (0.12) & 68.94 (0.19) & 0.88M \\
        \dipole   & 82.08 (0.29) & 70.35 (0.21) & 1.41M &  88.69 (0.24) & 66.22 (0.15) & 1.72M \\
        \timeline & 82.34 (0.31) & 71.03 (0.24) & 0.95M &  87.53 (0.13) & 66.07 (0.21) & 0.73M \\
        \gbert    & 81.50 (0.24) & 71.18 (0.12) & 3.58M &  87.26 (0.12) & 68.04 (0.17) & 3.95M \\
        \hitanet  & 82.77 (0.35) & 71.93 (0.29) & 2.08M &  88.10 (0.28) & 68.21 (0.33) & 2.39M  \\
        \cgl      & 84.19 (0.16) & 71.77 (0.10) & 0.55M &  89.05 (0.15) & 69.36 (0.22) & 0.60M  \\
        \midrule
        \modelname & \textbf{86.14 (0.14)} & \textbf{73.08 (0.09)} & 0.68M & \textbf{90.83 (0.09)} & \textbf{71.14 (0.15)} & 0.88M \\
        \bottomrule
    \end{tabular}
    \caption{Heart failure prediction results on MIMIC-III and MIMIC-IV using AUC (\%) and ${F_1}$ (\%)}
    \label{tab:result_hf}
\end{table*}

\subsubsection{{Baseline Methods}}
To compare {\modelname} with state-of-the-art models, we select the following methods as baselines:
\begin{itemize}
    \item RNN/Attention-based model: \retain~\cite{choi2016retain}, \dipole~\cite{ma2017dipole}, \timeline~\cite{bai2018interpretable}, and \hitanet~\cite{luo2020hitanet}.
    \item CNN-based model: \deepr~\cite{Phuoc2017deepr}.
    \item Graph-based model: \gram~\cite{choi2017gram}, \gbert~\cite{shang2019pretrain}, and \cgl~\cite{lu2021collaborative}.\footnote{We remove pretraining and medication in {\gbert} and clinical notes in CGL to ensure each baseline is trained with the same data.}
\end{itemize}
Detailed descriptions and parameter settings for baselines can be found in Appendix.

\subsubsection{Parameter Settings}
In our experiments, we randomly initialize model parameters. The hyper-parameters as well as activation functions are tuned on the validation set. Specifically, we set the threshold $\delta$ as $0.01$. The embedding size $s$ for $\mathbf{M, N}$ is 48, $s'$ for $\mathbf{R}$ is 32. The attention size $a$ is also 32. The hidden units $p$ of M-GRU and GRU are 256 on MIMIC-III and 350 on MIMIC-IV for the diagnosis prediction task. For the heart failure prediction task, $p$ is 100 on MIMIC-III and 150 on MIMIC-IV. When training {\modelname}, we use 200 epochs and the Adam~\cite{KingmaB14} optimizer. The learning rate is set as 0.01. All programs are implemented using Python 3.8.6 and PyTorch 1.7.1 with CUDA 11.1 on a machine with Intel i9-9900K CPU, 64GB memory, and Geforce RTX 2080 Ti GPU. The source code of {\modelname} can be found at \url{https://github.com/LuChang-CS/Chet/}.

\subsection{Diagnosis and Heart Failure Prediction Results}
We report the mean and standard deviations of evaluation metrics by running each model 5 times with different parameter initializations. We also report the parameter number (\# Params) of each model to evaluate the space complexity.

Table~\ref{tab:result_code} shows the results of diagnosis prediction using w-$F_1$ (\%) and R@$k$ (\%). Since the average diagnosis number in a visit is around 13, we set $k=[10, 20]$ for R@$k$. In Table~\ref{tab:result_code}, {\modelname} outperforms baselines on both datasets. It proves the effectiveness of global and local context with historical information from previous visits. It is worth noting that {\modelname} is better than {\gram} and {\cgl} even without medical ontology graphs used in these two models. It further validates the significance of learning disease combinations and transitions. We notice that {\gbert} does not achieve good w-$F_1$ scores. We infer it is mainly due to removing pretraining in the original model. In addition, all models have better w-$F_1$ and R@$k$ on MIMIC-IV than MIMIC-III. We conjecture the major reason is that MIMIC-IV has a larger training set.

Table~\ref{tab:result_hf} shows the results of heart failure prediction using AUC (\%) and $F_1$ (\%). Like Table~\ref{tab:result_code}, {\modelname} still performs the best. Moreover, we notice that the AUC on MIMIC-IV is higher than MIMIC-III, while the $F_1$ score is lower for all models. We infer that it is mainly because samples are more imbalanced in the test set of MIMIC-IV.

In addition, {\modelname} achieves good prediction scores even with fewer parameters than a majority of baselines. It also demonstrates the efficiency of {\modelname} when the numbers of patients, diseases, and visits in EHR datasets increase.

\subsection{Analysis for Diagnosis Prediction}
In medical practice, emerging diseases refer to new diseases in the future visit $T + 1$ that have not been diagnosed in the visit $T$. In the diagnosis prediction, it is natural for a model to predict diseases that have been diagnosed in the previous visit. However, a generic model should give consideration to both occurred diseases and emerging diseases for better risk predictions. Therefore, we evaluate the results of predicting emerging diseases to measure the ability of a model to capture the development scheme of diseases.

In this experiment, we choose MIMIC-III and divide the ground truth of the diagnosis prediction into two parts: persistent diseases and emerging diseases (emerging neighbor/unrelated), and thus obtain R@$k$ for each part respectively. We plot the prediction results including the mean and standard deviation of 5 runs on two parts in \figurename{~\ref{fig:occurred}}. Here, we select $k=[10, 20, 30, 40]$ to further analyze the trend of R@$k$ when $k$ increases. Note that the sum of mean values for two recall values on predicting persistent and emerging diseases should be equal to the corresponding R@$k$ in Table~\ref{tab:result_code}. We observe from \figurename{~\ref{fig:occurred}} that {\modelname} achieves the best recall values when predicting both persistent and emerging diseases. Specifically, compared to Table~\ref{tab:result_code}, we notice that the main improvement of {\modelname} over baselines results from predicting emerging diseases. We conclude that the combination graph and transition processes are beneficial for predicting emerging diseases but baselines fails to effectively utilize them. In addition, as $k$ increases, R@$k$ of emerging diseases is gradually close to R@$k$ of persistent diseases. It shows that emerging diseases play a more important role when considering a larger set of potential diagnoses in future visits. Therefore, it proves the importance of exploring the development scheme of diseases.

\begin{figure*}
    \centering
    \includegraphics[width=\linewidth]{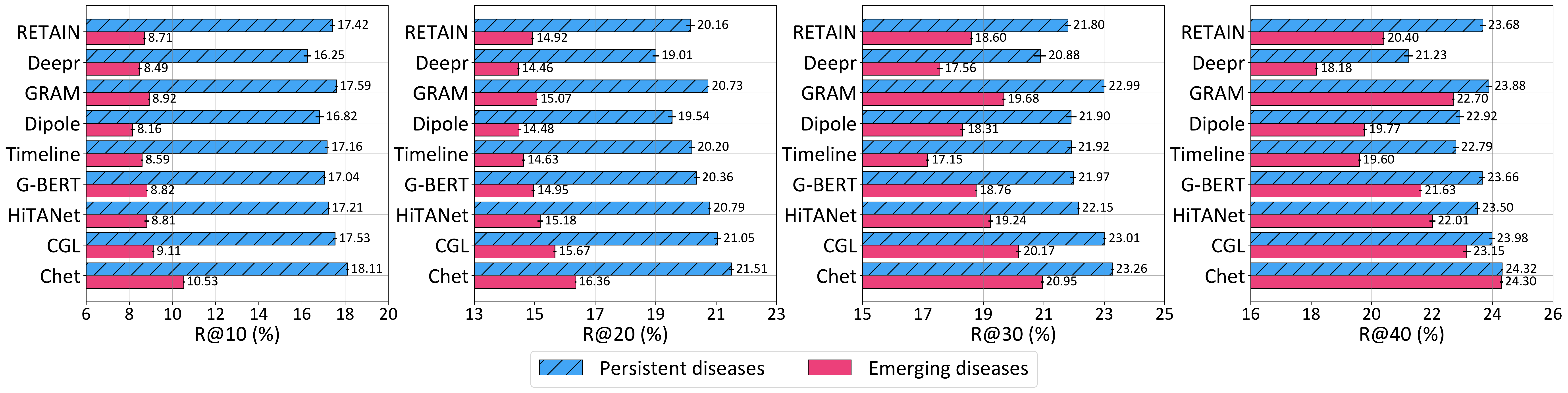}
    \caption{R@${k}$ of predicting persistent/emerging diseases for diagnosis prediction on the MIMIC-III dataset.}
    \label{fig:occurred}
\end{figure*}

\subsection{Ablation Study}
To further analyze the effectiveness of each module in {\modelname}, we also compare two variants of {\modelname}:
\begin{itemize}
    \item {\modelname}$_{d\text{-}}$: We remove the dynamic part of GNN in {\modelname}. Instead of using dynamic subgraphs $\mathcal{M}^t, \mathcal{B}^t, \hat{\mathcal{B}}^t, \mathcal{N}^t$ in Equations~(\ref{eq:graph1}) and (\ref{eq:graph2}), we adopt a universal embedding matrix for all diseases and use the global combination graph $\mathcal{A}$ for the aggregation of diagnoses and neighbors. In Equation~(\ref{eq:trans_2}), we project the universal embeddings into the dimension $p$ with a fully connected layer to replace the embeddings $\mathbf{R}$ for unrelated disease.
    \item {\modelname}$_{t\text{-}}$: We remove the transition functions in {\modelname}. Similar to {\gram} and {\dipole}, we use a direct sum of diagnosis embeddings $\mathbf{H}^t_D$ calculated by the GNN. In the meantime, we retain the dynamic subgraphs in Equation~(\ref{eq:graph1}) and remove Equation~(\ref{eq:graph2}), since it is no longer needed after removing the transition functions.
\end{itemize}
Table~\ref{tab:ablation} shows the results of diagnosis and heart failure prediction for {\modelname}, {\modelname}$_{d\text{-}}$, and {\modelname}$_{t\text{-}}$ on MIMIC-III. We report the mean value of 5 runs. We observe that removing the dynamic part of GNN and transition functions both lead to a decrease of the $F_1$, recall, and AUC values. It validates the effectiveness of dynamic learning for the disease combination graph and transition processes. Among these two variants, {\modelname}$_{d\text{-}}$ performs slightly better than {\modelname}$_{t\text{-}}$. We infer that it is because {\modelname}$_{d\text{-}}$ retains the general structure of {\modelname}, including the disease combination structure and transition functions. Moreover, after removing transition functions, {\modelname}$_{t\text{-}}$ has similar results to {\retain} and {\gram} in Tables~\ref{tab:result_code} and \ref{tab:result_hf}, which shows the significance of exploring the development scheme of diseases by transition functions.

\begin{table}
    \centering
    \small
    \begin{tabular}{lccc|cc}
        \toprule
        \multirow{2}{*}{\textbf{Models}} & \multicolumn{3}{c}{\textbf{Diagnosis}} &\multicolumn{2}{c}{\textbf{Heart failure}}\\
        \cmidrule{2-6}
        & \textbf{w-}$\boldsymbol{F_1}$ & \textbf{R@10} & \textbf{R@20} & \textbf{AUC} & $\boldsymbol{F_1}$ \\
        \midrule
        {\modelname}$_{d\text{-}}$ & 21.97 & 27.09 & 36.26 & 84.96 & 72.07 \\
        {\modelname}$_{t\text{-}}$ & 21.16 & 26.72 & 35.91 & 84.01 & 71.93 \\
        \modelname                 & \textbf{22.63} & \textbf{28.64} & \textbf{37.87} & \textbf{86.14} & \textbf{73.08} \\
        \bottomrule
    \end{tabular}
    \caption{Diagnosis prediction and heart failure prediction for {\modelname} variants on the MIMIC-III dataset.}
    \label{tab:ablation}
\end{table}
\section{Related Work}
\label{sec:related_works}
\paragraph{Health event prediction:} Deep learning methods have been widely adopted to predict health events. Choi \textit{et al.}~\cite{choi2016doctor} proposed DoctorAI by applying GRU to visit sequences to predict future diagnoses and time duration of the next visit. Choi \textit{et al.}~\cite{choi2016retain} proposed {\retain} to predict heart failure and provide interpretation for predictions. It combines two RNNs of different directions with attention. \deepr is proposed by Nguyen \textit{et al}.~\cite{Phuoc2017deepr} to predict readmission of patients in the next three months. It regards diagnoses in a visit as a sentence and applies one-dimensional CNN to extract features. Ma \textit{et al}.~\cite{ma2017dipole} proposed {\dipole} by applying various attention mechanisms on a bi-directional RNN to predict diagnoses of the next visit. Bai \textit{et al}.~\cite{bai2018interpretable} considered the time duration between two visits and proposed the Timeline model.  Luo \textit{et al}.~\cite{luo2020hitanet} used a self-attention-based method to encode the time information and detect key time steps among patients’ historical visits. However, these models do not utilize disease combination information in EHR data and thus cannot mine hidden disease patterns which help to predict future diagnoses.

\paragraph{Graph structure in healthcare:} Recently, graph structures have also shown effectiveness to model EHR data. {\gram}~\cite{choi2017gram} and {\gbert}~\cite{shang2019pretrain} both construct a medical ontology graph based on disease domain knowledge. Choi \textit{et al}. proposed GCT~\cite{gct_aaai20}, a graph convolutional transformer by constructing a graph of diagnoses, treatments, and lab results. Lu \textit{et al}.~\cite{lu2021collaborative} considered horizontal links in the medical ontology graph and constructed a patient-disease graph to learn hidden disease relations. Although these methods utilize various graph structures of EHR data, they do not model the development scheme of diseases.

\paragraph{Context learning in healthcare:} Lee \textit{et al}.~\cite{lee2018diagnosis} proposed MCA-RNN with demographics as context for patients. Shang \textit{et al}. regarded a visit as a sentence and applied a modified version of BERT~\cite{devlin2018bert}, {\gbert}, by removing position embeddings. Ma \textit{et al}.~\cite{ma2020concare} proposed ConCare by applying self-attention on visit sequences to extract personal context. However, these methods do not consider the dynamics of diseases and fail to capture the context in disease development.
\section{Conclusion}\label{sec:conclusion}
In this paper, we study challenges in health event predictions of utilizing disease combinations and exploring the development scheme of diseases. To address these challenges, we propose {\modelname} by designing a context-aware dynamic graph learning method to learn disease combinations and transition functions to model disease developments. A global disease combination graph is constructed for all diseases. Three dynamic subgraphs are extracted from the global graph for each visit. We adopt a dynamic graph neural network to learn global and local context from subgraphs of a visit and propose three transition functions to integrate historical context from previous visits. Experimental results on two real-world EHR datasets demonstrate the effectiveness of our proposed method on two health event prediction tasks. We also provide detailed analysis towards the diagnosis prediction results and validate the significance of learning disease combinations and transition processes. In the future, we will incorporate other data types in EHR, adopt medical ontology graphs as external knowledge, and explore effective ways to provide more interpretability in health event predictions.

\section*{Acknowledgements}
This work is supported in part by the US National Science
Foundation under grants 1948432 and 2047843. Any opinions,
findings, and conclusions or recommendations expressed in
this material are those of the authors and do not necessarily
reflect the views of the National Science Foundation.

\bibliography{ref}

\clearpage
\newpage
\appendix
\section{Subgraphs' Adjacency Matrix Calculation}
In our settings, the dimension of $\mathcal{M}^t, \mathcal{B}^t$, $\hat{\mathcal{B}}^t$, and $\mathcal{N}^t$ is ${d \times d}$. However, the disease number $d$ can be large in real practice. For example, there are nearly 13,000 medical codes ($d \sim 13,000$) in ICD-9-CM. In a real-world EHR dataset, MIMIC-III, there are still nearly 5,000 ($d \sim 5,000$) medical codes. When the patient number $|\mathcal{U}|$ and $d$ are large, there exist storage/memory problems even with sparse matrices:
\begin{itemize}
    \item \textbf{Storing the entire EHR dataset}. The dimension of the entire EHR dataset is $|\mathcal{U}| \times T \times d \times d$, if we store adjacency matrices for all visits. When $|\mathcal{U}|$ and $d$ increase, the size of the dataset will increase rapidly.
    \item \textbf{Training models with mini-batches}. When training a deep learning model, assuming the batch size is $b$, we have to load four batched matrices, $\mathcal{M}^t, \mathcal{B}^t$, $\hat{\mathcal{B}}^t$, and $\mathcal{N}^t$ whose dimensions are $b \times T \times d \times d$, into CPU/GPU memory, but it is usually not applicable when $d$ is large.
\end{itemize}
Based on the above analysis, we notice that the main bottlenecks are dynamic matrices $\mathcal{M}^t, \mathcal{B}^t$, $\hat{\mathcal{B}}^t$, and $\mathcal{N}^t$. Therefore, we consider replacing them with the diagnosis/neighbor vectors $\mathbf{m}^t, \mathbf{n}^t$ and the static adjacency matrix $\mathcal{A}$. 
\begin{theorem}
    \label{thm:a}
    The adjacency matrices $\mathcal{M}^t, \mathcal{B}^t$, $\hat{\mathcal{B}}^t$, and $\mathcal{N}^t$ for a visit $t$ can be represented as
    \begin{align}
        \mathcal{M}^t &= \mathbf{m}^t \odot \mathcal{A} \odot {\mathbf{m}^t}^\top, \notag \\
        \mathcal{B}^t &= \mathbf{m}^t \odot \mathcal{A} \odot {\mathbf{n}^t}^\top, \notag \\
        \hat{\mathcal{B}}^t &= \mathbf{n}^t \odot \mathcal{A} \odot {\mathbf{m}^t}^\top, \notag \\
        \mathcal{N}^t &= \mathbf{n}^t \odot \mathcal{A} \odot {\mathbf{n}^t}^\top.\notag 
    \end{align}
\end{theorem}
\begin{proof}
    Here, we only give proof for $\mathcal{B}^t = \mathbf{m}^t \odot \mathcal{A} \odot {\mathbf{n}^t}^\top$. The other three equations can be proven in a similar manner.
    
    For $\forall c_i, c_j \in \mathcal{C}$, when $\mathcal{B}^t_{ij} = 1$, according to the definition of $\mathcal{B}^t$, we can infer that $c_i$ is a diagnosis, and $c_j$ is a neighbor of $c_i$, and $\mathbf{m}^t_i = 1$, ${\mathbf{n}^t_j} = 1$, and $\mathcal{A}_{ij} = 1$. Therefore, $(\mathbf{m}^t \odot \mathcal{A} \odot {\mathbf{n}^t}^\top)_{ij} = 1$. When $\mathcal{B}^t_{ij} = 0$, there are two cases:
    \begin{enumerate}
        \item $c_i$ is not a diagnosis node or $c_j$ is not a neighbor node in visit $t$. In this case, we have $\mathbf{m}^t_i = 0$ or ${\mathbf{n}^t_j} = 0$. Therefore, we can deduce $(\mathbf{m}^t \odot \mathcal{A} \odot {\mathbf{n}^t}^\top)_{ij} = 0$.
        \item $c_i$ is a diagnosis node and $c_j$ is not a neighbor of $c_i$ but a neighbor of other diagnoses. We can infer $\mathcal{A}_{ij} = 0$. Therefore, we can also get $(\mathbf{m}^t \odot \mathcal{A} \odot {\mathbf{n}^t}^\top)_{ij} = 0$.
    \end{enumerate}
    Since $\mathcal{B}^t_{ij} = (\mathbf{m}^t \odot \mathcal{A} \odot {\mathbf{n}^t}^\top)_{ij}$ holds for all pairs, we can conclude that $\mathcal{B}^t = \mathbf{m}^t \odot \mathcal{A} \odot {\mathbf{n}^t}^\top$.
\end{proof}
Based on Theorem~\ref{thm:a}, when storing an EHR dataset, we only need to store $\mathbf{m}^t$ and $\mathbf{n}^t$ for each visit, without keeping four adjacency matrices. However, when training models, if we directly calculate $\mathbf{m}^t \odot \mathcal{A} \odot {\mathbf{n}^t}^\top$ and other three matrices, it will bring a matrix whose dimension is still $b \times T \times d \times d$, because $\mathbf{m}^t \odot \mathcal{A} \odot {\mathbf{n}^t}^\top \in \mathbb{R}^{d \times d}$ is different for different visits and patients. Therefore, we need more derivations for aggregations in the graph layer beyond manipulating adjacency matrices.

\begin{theorem}
    \label{thm:b}
    The adjacency matrices for subgraphs in neighborhood aggregation, i.e., Equations~(\ref{eq:graph1}) and (\ref{eq:graph2}), can be replaced by $\mathcal{A}$ without introducing matrices of dimensions $b \times T \times d \times d$.
\end{theorem}
\begin{proof}
    We also only give proof for the item $\mathcal{B}^t\mathbf{N}$. According to Theorem~\ref{thm:a}, we have
    \begin{align}
        \mathcal{B}^t\mathbf{N} &= (\mathbf{m}^t \odot \mathcal{A} \odot {\mathbf{n}^t}^\top)\mathbf{N} \notag \\
        &= (\mathbf{m}^t \odot \mathcal{A}\mathbf{D}_{{\mathbf{n}^t}})\mathbf{N} \notag \\
        &= (\mathbf{m}^t \odot \mathcal{A})(\mathbf{D}_{{\mathbf{n}^t}}\mathbf{N}) \notag \\
        &= (\mathbf{m}^t \odot \underbrace{(\mathcal{A}\underbrace{(\mathbf{n}^t \odot \mathbf{N})}_{d \times s})}_{d\times s}. \notag
    \end{align}
    Here, $\mathbf{D}_{{\mathbf{n}^t}}$ is a diagonal matrix based on $\mathbf{n}^t$. The dimension of $\mathbf{n}^t  \odot \mathbf{N}$ is $d \times s$. After multiplying $\mathcal{A}$, the dimension is still $d \times s$. Since $\mathcal{A}$ is the same along all visits and patients, this equation only has matrices whose dimensions are $b \times T \times d \times s$ during training models. Therefore, when $s \ll d$, we can avoid introducing matrices of which dimensions are $b \times T \times d \times d$.
\end{proof}

By applying Theorems~\ref{thm:a} and \ref{thm:b}, we can optimize the storage/memory problems. The optimized graph layer for Equations~(\ref{eq:graph1}) and (\ref{eq:graph2}) is summarized as follows:
\begin{align}
    \mathbf{Z}^t_D &= \mathbf{m}^t \odot (\mathbf{M} + \mathcal{A}(\mathbf{m}^t \odot \mathbf{M}) + \mathcal{A}(\mathbf{n}^t \odot \mathbf{N})), \notag \\
    \mathbf{Z}^t_N &= \mathbf{n}^t \odot (\mathbf{N} + \mathcal{A}(\mathbf{n}^t \odot \mathbf{N}) + \mathcal{A}(\mathbf{m}^t \odot \mathbf{M})). \notag
\end{align}

\section{Pseudo-code of \modelname}
\begin{algorithm}[t]
    \caption{{\modelname}($\mathcal{A}$, $r_u$, $\mathbf{M, N, R}$)}\label{algo:model_code}
    \newcommand\mycommfont[1]{\footnotesize\ttfamily\textcolor{blue}{#1}}
    \SetCommentSty{mycommfont}
    \DontPrintSemicolon
    \SetAlgoLined
    \SetKwInOut{Input}{Input}
    \SetKwInOut{Output}{Output}
    \Input{Adjacency matrix $\mathcal{A}$ for the combination graph; \newline
        A visit sequence $r_u$ for a patient; \newline
        Embedding matrices $\mathbf{M, N, R}$ for diseases}
    \Output{The patient embedding $\mathbf{o}$}
    $\mathbf{m}^1, \mathbf{m}^2, \dots, \mathbf{m}^T \leftarrow$ Retrieve diagnoses in each visit from $r_u$ \\
    $\mathbf{n}^1, \mathbf{n}^2, \dots, \mathbf{n}^T \leftarrow$ Calculate neighbors based on the combination graph and diagnoses \\
    $\mathbf{h}_\text{p}^0 \leftarrow \mathbf{{0}}$ \\
    \For{$t \leftarrow 1$ to $T$}{
        \tcp{Optimized dynamic graph layer}
        $\mathbf{Z}^t_{\{D, N\}} \leftarrow$ Aggregate global/local context with the optimized graph layer with $\mathbf{M, N}$, and $\mathcal{A}$ \\
        $\mathbf{H}^t_{\{D,N\}} \leftarrow$ Calculate hidden embeddings for diagnoses and neighbors \\
        \tcp{Transition functions}
        $\mathbf{m}^t_\text{p}, \mathbf{m}^t_{\text{en}}, \mathbf{m}^t_{\text{eu}} \leftarrow$ Divide diagnosis roles from $\mathbf{m}^t$ \\
        \eIf{$t = 1$}{
            $\mathbf{h}^1_\text{p} \leftarrow \text{GRU}(\mathbf{m}_\text{p}^1 \odot \mathbf{H}^1_D, \mathbf{h}^0_\text{p})$ \\
        }
        {
            $\mathbf{h}_{\text{en}}^t \leftarrow \text{Attention}(\mathbf{m}_{\text{en}}^t \odot \mathbf{H}^{t-1}_N, \mathbf{m}_{\text{en}}^t \odot \mathbf{H}^{t-1}_N, \mathbf{m}_{\text{en}}^t \odot \mathbf{H}^t_D)$ \\
            $\mathbf{h}_{\text{eu}}^t \leftarrow \text{Attention}(\mathbf{m}_{\text{eu}}^t \odot \mathbf{R}, \mathbf{m}_{\text{eu}}^t \odot \mathbf{R}, \mathbf{m}_{\text{eu}}^t \odot \mathbf{H}^t_D)$ \\
            $\mathbf{h}^t_\text{p} \leftarrow \text{M-GRU}(\mathbf{m}_\text{p}^t \odot \mathbf{H}^t_D, \mathbf{h}_{\text{en}}^t, \mathbf{h}_{\text{eu}}^t, \mathbf{h}^{t - 1}_\text{p})$ \\
        }
        $\mathbf{v}^t \leftarrow \text{max\_pooling}(\mathbf{h}^t_\text{p})$ \\
    }
    $\mathbf{o} \leftarrow$ Calculate a weighted sum of $\mathbf{v}^1, \mathbf{v}^2, \dots, \mathbf{v}^T$ with a location-based attention \\
    \Return $\mathbf{o}$ \\
\end{algorithm}

In summary, we demonstrate the pseudo-code of {\modelname} including the dynamic graph learning and temporal learning in Algorithm~\ref{algo:model_code}. For each patient, the inputs of {\modelname} are a pre-constructed global combination graph of diseases, visit records of this patient, and initialized embedding matrices for diseases as diagnoses, neighbors, and unrelated diseases. At lines 5-6, we execute the optimized dynamic graph learning for each visit. At lines 7-14, transition functions are conducted for different diagnosis roles. Finally, we calculate the patient embedding as the model output of {\modelname}.

\section{Baselines}
To compare {\modelname} with state-of-the-art models, we select the following methods as baselines:
\begin{itemize}
    \item \retain~\cite{choi2016retain}: A network containing two RNNs of different directions with an attention method.
    \item \deepr~\cite{Phuoc2017deepr}: A one-dimensional CNN with max pooling.
    \item \gram~\cite{choi2017gram}: An RNN model with an attention method on the medical ontology graph.
    \item \dipole~\cite{ma2017dipole}: A bi-directional RNN model with attention mechanisms.
    \item \timeline~\cite{bai2018interpretable}: An RNN model with an attention method and time decay function on visit durations.
    \item \gbert~\cite{shang2019pretrain}: A BERT-based model with attention on the medical ontology graphs of diseases and medications. Here, we modify {\gbert} by removing the pre-training process and the medication module to make a fair comparison.
    \item \hitanet~\cite{luo2020hitanet}: A Transformer-based model considering time intervals between admissions. The inputs are multi-hot vectors.
    \item \cgl~\cite{lu2021collaborative}: An RNN model with a medical ontology graph and a patient-disease graph. Here, we do not use clinical notes in {\cgl} for a fair comparison.
\end{itemize}
Note that, we do not compare with traditional machine learning methods such as logistic regression because they have been proven not good in \retain~\cite{choi2016retain} and \deepr~\cite{Phuoc2017deepr}.

Detailed parameter settings for baselines are given as follows:
\begin{itemize}
    \item \retain: The embedding size for visits is 256. Hidden units for two RNN layers are 128.
    \item \deepr: The embedding size for medical codes is 100. The kernel size and filter number for an 1-D CNN layer are 3 and 128, respectively.
    \item \gram: The embedding size for medical codes is 100. The attention size is 100. Hidden units for RNN is 128.
    \item \dipole: The embedding size for visits is 256. We choose concatenation-based attention and the attention size is 128. Hidden units for RNN is 128.
    \item \timeline: The embedding size for medical codes is 100. The attention size is 100.  Hidden units for RNN is 128.
    \item \gbert: The parameters are the same as \cite{shang2019pretrain}. (1) GAT: The embedding size for medical codes is 75, the number of attention heads is 4; (2) BERT: the hidden size is 300. The position-wise feed-forward networks include 2 hidden layers with 4 attention heads for each layer. The dimension of each hidden layer is 300.
    \item \hitanet: The parameter are the same as \cite{luo2020hitanet}. (1) {\hitanet}: the dense space size for diseases is 256. The space size for time interval, query vector size, and latent space size for time interval are 64. (2) Transformer: the attention size is 64. The number of attention heads is 4. The size of middle feed-forward network is 1024.
    \item \cgl: The embedding size for diseases and patients are 32 and 16. The graph layer number is 2. The hidden size for patients and diseases are 32, 64, and 128. Hidden units for RNN are 200.
\end{itemize}

\end{document}